\documentclass[letterpaper, 10 pt, conference]{ieeeconf}  

\IEEEoverridecommandlockouts{}                              
\usepackage{fix-cm}
\usepackage{etex}

\usepackage{dblfloatfix}

\usepackage{nag}

\makeatletter
\@ifpackageloaded{xcolor}{}{%
\usepackage[table,x11names,dvipsnames,svgnames]{xcolor}%
}
\makeatother

\usepackage{colortbl}

\usepackage{graphicx}
\usepackage{wrapfig}

\definecolorset{RGB}{lyft}{}{Red,194,39,36;Sunset,202,53,33;Orange,205,68,20;Amber,200,117,42;Yellow,242,169,52;Citron,186,188,44;Lime,112,159,33;Green,56,139,31;Mint,45,118,56;Teal,52,133,135;Cyan,60,132,202;Blue,55,94,248;Indigo,64,13,247;Purple,115,42,248;Pink,176,25,145;Rose,176,32,75}

\usepackage{cite}

\usepackage{microtype}

\usepackage[american]{babel}

\usepackage{array}
\usepackage{multirow}
\usepackage{booktabs}
\usepackage{makecell} 

\ifcsname labelindent\endcsname
\let\labelindent\relax
\fi
\usepackage[inline]{enumitem}

\usepackage{subfig}

\newenvironment{lenumerate}[2][]
{\begin{enumerate}[label=(#2\arabic*),leftmargin=0.2in,itemindent=0.15in,#1]}
{\end{enumerate}}

\setlist*[enumerate,1]{label={\itshape\arabic*)}}

\makeatletter
\newcommand{\paragraphswithstop}{%
\let\copyparagraph\paragraph%
\renewcommand\paragraph[1]{\copyparagraph{##1.}}%
}
\makeatother

\usepackage[framemethod=tikz]{mdframed}

\makeatletter
\def\namedlabel#1#2{\begingroup
  #2%
  \def\@currentlabel{#2}%
  \phantomsection\label{#1}\endgroup
}
\makeatother

\usepackage{suffix}

\usepackage{environ}

\makeatletter
\newsavebox{\boxifnotempty}
\newcommand{\displayifnotempty}[3]{\sbox\boxifnotempty{#2}\setbox0=\hbox{\usebox{\boxifnotempty}\unskip}%
\ifdim\wd0=0pt
\else
 #1\usebox{\boxifnotempty}#3%
\fi%
}
\newcommand{\ifempty}[2]{\setbox0=\hbox{#1\unskip}%
\ifdim\wd0=0pt%
 #2%
\fi%
}
\newcommand{\ifnotempty}[2]{\setbox0=\hbox{#1\unskip}%
\ifdim\wd0>0pt%
 #2%
\fi%
}
\makeatother

\usepackage{algorithm}

\usepackage{scrlfile}

\makeatletter
\newcommand*\newstoreddef[1]{
  \BeforeClosingMainAux{%
    \immediate\write\@auxout{%
      \string\restoredef{#1}{\csname #1\endcsname}%
    }%
  }%
}
\newcommand*{\restoredef}[2]{
  \expandafter\gdef\csname stored@#1\endcsname{#2}%
}
\newcommand*{\storeddef}[1]{
  \@ifundefined{stored@#1}{0}{\csname stored@#1\endcsname}%
}
\makeatother

\usepackage{pageslts}
\NewEnviron{tee}{\BODY\typeout{Marker Tee [start] ^^J \BODY ^^JMaker Tee [end]}}

\input{preamble/math}
\newcommand{\real}[1]{\mathbb{R}^{#1}{}}

\newcommand{\bmat}[1]{\begin{bmatrix}#1\end{bmatrix}}

\newcommand{\transpose}{^\mathrm{T}}

\newcommand{\defeq}{\doteq}

\DeclarePairedDelimiter{\abs}{\lvert}{\rvert}

\DeclarePairedDelimiter{\norm}{\lVert}{\rVert}

\newcommand{\vct}[1]{\mathbf{#1}}

\DeclareMathOperator{\proj}{proj}

\DeclareMathOperator{\stack}{stack}

\newcommand{\subjectto}{\textrm{subject to}\;}

\newcommand{\ijE}[1][]{(i,j) \in E_{#1}}

\providecommand{\cC}{\mathcal{C}}

\providecommand{\cE}{\mathcal{E}}

\providecommand{\cI}{\mathcal{I}}

\providecommand{\cL}{\mathcal{L}}

\providecommand{\cO}{\mathcal{O}}

\providecommand{\cT}{\mathcal{T}}
\providecommand{\cU}{\mathcal{U}}
\providecommand{\cV}{\mathcal{V}}

\providecommand{\cX}{\mathcal{X}}

\usepackage{units}

  \newcommand{\newcolorlabel}[2]{%
  \expandafter\newcommand\csname #1\endcsname[1]{%
    \tikz[baseline]{\node[text=white,fill=#2,anchor=base,text height=1.3ex,text depth=0.1ex,font=\sffamily\bfseries]{##1}}}%
}

\newcommand{\newcommenter}[2]{%
  \expandafter\newcommand\csname #1\endcsname[1]{%
    \fcolorbox{#2}{#2}{\color{white}\textsf{\textbf{#1}}}
    {\color{#2}##1}}%
  \expandafter\newcommand\csname at#1\endcsname{%
    \fcolorbox{#2}{#2}{\color{white}\textsf{\textbf{@#1}}}
    {\color{#2}}}%
  \expandafter\newcommand\csname #1hl\endcsname[2]{%
    \colorbox{#2}{\color{white}\textsf{\textbf{#1}}}\sethlcolor{Azure2}\hl{##2}~%
    \expandafter\ifx\csname commentarrow\endcsname\relax$\leftarrow$\else \commentarrow[#2]\fi~%
    {\color{#2}##1}}%
  \expandafter\newcommand\csname #1st\endcsname[2]{%
    \colorbox{#2}{\color{white}\textsf{\textbf{#1}}}\sout{##2}~%
    \expandafter\ifx\csname commentarrow\endcsname\relax$\leftarrow$\else \commentarrow[#2]\fi~%
    {\color{#2}##1}}%
}
\newcommenter{TODO}{DodgerBlue1}
\newcommenter{rtron}{Green3}

\usepackage{comment}

\usepackage{pdfcomment}

\usepackage{soul}

\usepackage[normalem]{ulem}

\usepackage{csquotes}

\usepackage{tikz}
\usetikzlibrary{calc}
\usetikzlibrary{matrix}
\usetikzlibrary{chains}
\usetikzlibrary{shapes.geometric}
\usetikzlibrary{arrows.meta}
\usetikzlibrary{decorations.pathreplacing}
\usetikzlibrary{backgrounds}

\tikzset{
  dim above/.style={to path={\pgfextra{
        \pgfinterruptpath
        \draw[>=latex,|->|] let
        \p1=($(\tikztostart)!1.5em!90:(\tikztotarget)$),
        \p2=($(\tikztotarget)!1.5em!-90:(\tikztostart)$)
        in(\p1) -- (\p2) node[pos=.5,sloped,above]{#1};
        \endpgfinterruptpath
      }
    }
  },
  dim double above/.style={to path={\pgfextra{
        \pgfinterruptpath
        \draw[>=latex,|->|] let
        \p1=($(\tikztostart)!3em!90:(\tikztotarget)$),
        \p2=($(\tikztotarget)!3em!-90:(\tikztostart)$)
        in(\p1) -- (\p2) node[pos=.5,sloped,above]{#1};
        \endpgfinterruptpath
      }
    }
  },
  dim below/.style={to path={\pgfextra{
        \pgfinterruptpath
        \draw[>=latex,|->|] let 
        \p1=($(\tikztostart)!-1em!-90:(\tikztotarget)$),
        \p2=($(\tikztotarget)!-1em!90:(\tikztostart)$)
        in (\p1) -- (\p2) node[pos=.5,sloped,below]{#1};
        \endpgfinterruptpath
      }
    }
  },
}

\tikzset{
    right angle quadrant/.code={
        \pgfmathsetmacro\quadranta{{1,1,-1,-1}[#1-1]}     
        \pgfmathsetmacro\quadrantb{{1,-1,-1,1}[#1-1]}},
    right angle quadrant=1, 
    right angle length/.code={\def\rightanglelength{#1}},   
    right angle length=2ex, 
    right angle symbol/.style n args={3}{
        insert path={
            let \p0 = ($(#1)!(#3)!(#2)$) in     
                let \p1 = ($(\p0)!\quadranta*\rightanglelength!(#3)$), 
                \p2 = ($(\p0)!\quadrantb*\rightanglelength!(#2)$) in 
                let \p3 = ($(\p1)+(\p2)-(\p0)$) in  
            (\p1) -- (\p3) -- (\p2)
        }
    }
}

\newcommand{\pgfextractangle}[3]{%
    \pgfmathanglebetweenpoints{\pgfpointanchor{#2}{center}}
                              {\pgfpointanchor{#3}{center}}
    \global\let#1\pgfmathresult  
}

\usetikzlibrary{shapes.arrows}
\newcommand{\commentarrow}[1][Azure4]{\tikz[baseline=-3pt]{\node[shape border uses incircle, fill=#1,rotate=180,single arrow, inner sep=1pt, minimum size=6pt, single arrow head extend=2pt]{};}}

\tikzset{ax/.style={-latex,line width=2pt}}

\tikzset{camera/.style={fill=Sienna1,fill opacity=0.5},%
image plane/.style={draw=RoyalBlue3,line width=2pt}}

\usepackage{tkz-euclide}
\newcommenter{mmitjans}{Emerald}
\newcommenter{mahroo}{BurntOrange}
\newcommand{\rrt}{{\texttt{RRT$^*$}}}
\newcommand{\rrtstar}{{\texttt{RRT$^*$}}}
\newcommand{\rrtfunc}[1]{{\fontfamily{qcr}\selectfont #1}}
\usepackage[dvipsnames]{xcolor}
\usepackage{subfig}

\newsavebox{\bigleftbox}

\title{\LARGE \bf
  Robust Sample-Based Output-Feedback  Path Planning}
 
\author{Mahroo Bahreinian$^{1}$, Marc Mitjans$^{2}$, and Roberto Tron$^{3}$

  \thanks{This work was supported by ONR MURI N00014-19-1-2571 ``Neuro-Autonomy: Neuroscience-Inspired Perception, Navigation, and Spatial Awareness''}
  \thanks{$^{1}$Division of Systems Engineering at Boston University, Boston, MA, 02215 USA.
    \{\tt\small mahroobh@bu.edu\}}%
  \thanks{$^{2}$Department of Mechanical Engineering at Boston University, Boston, MA, 02215 USA. M. Mitjans is additionally supported by ”la Caixa” Foundation fellowship LCF/BQ/AA18/11680117.
    \{\tt\small mmitjans@bu.edu\}}%
  \thanks{$^{3}$Department of Mechanical and System Engineering.\{\tt\small tron@bu.edu\}}%
}

\usepackage{algorithm}

\usepackage{algpseudocode}
\begin{document}

\maketitle
\thispagestyle{empty}
\pagestyle{empty}

\begin{abstract}
  We propose a novel approach for sampling-based and control-based motion planning that combines a representation of the environment obtained via modified version of optimal Rapidly-exploring Random Trees (\rrtstar{}), with landmark-based output-feedback controllers obtained via Control Lyapunov Functions, Control Barrier Functions, and robust Linear Programming. Our solution inherits many benefits of \rrtstar-like algorithms, such as the ability to implicitly handle arbitrarily complex obstacles, and asymptotic optimality. Additionally, it extends planning beyond the discrete nominal paths, as feedback controllers can correct deviations from such paths, and are robust to discrepancies between the map used for planning and the real environment.
  We test our algorithms first in simulations and then in experiments, testing the robustness of the approach to practical conditions, such as deformations of the environment, mismatches in the dynamical model of the robot, and measurements acquired with a camera with a limited field of view.
\end{abstract}

\section{INTRODUCTION}
The problem of motion planning from an initial state toward a goal state has received great attention in mobile robotics. One of the currently most popular techniques for solving this problem is represented by sampling-based algorithms, where the planner is not given an explicit representation of the environment (e.g., via polygons), but instead uses a \emph{sampling function} that can be used to query whether an arbitrary point is in free space or inside an obstacle. Together with a \emph{steering function} that can find trajectories between samples, algorithms such as \rrt{} and derivatives build a \emph{tree} that is rooted at the goal location and that extends toward every reachable point in the free space. When such tree arrives at a given starting location, a nominal feasible path can be found by tracing it back along the tree to the root.
However, in practice, following this path requires full knowledge of the position in the environment and a lower-level control to compensate for disturbances. Additionally, although the sampling process naturally reveals information about the obstacles via sample collisions, such information is typically discarded after planning; relatedly, traditional approaches do not address the fact that there might be discrepancies between the implicit map given by the sampling function and the real free configuration space.

In this paper we take advantage of the capabilities of \rrt-like algorithms to effectively represent the free configuration space, but augment it with linear output-feedback controllers that guide the state along the edges of the tree graph based on the observation of \emph{landmarks}, points whose location is known in the map and that can be easily recognized (but can be generally distinct from the obstacles or generated samples). Our output-feedback controllers provide remedies to the three aforementioned shortcomings of traditional methods:
\begin{enumerate*}
\item it enables us to simplify the tree representation (i.e., reduce the number of nodes) while also extending it to regions that were not explicitly sampled;
\item it steers clear of obstacles (within the resolution limits given by the finite sampling) by explicitly avoiding samples that were found in collision; and
\item it provides robustness to discrepancies in the map used for the planning that are reflected in the landmarks (if the actual landmark locations are somewhat different, the resulting control will change accordingly and without replanning).
\end{enumerate*}

\subsection{Review of prior work}
Sampling-based planning algorithms, such as Probabilistic Road Map \cite{kavraki1996probabilistic}, Rapidly exploring Random Tree (\rrt) \cite{rrt,lavalle2001randomized} and asymptotically optimal Rapidly Exploring Random Tree (\rrtstar{}), \cite{karaman2011sampling}, have become popular in the last few years due to their good practical performance, and their probabilistic completeness \cite{lavalle2006planning,lavalle2001randomized,karaman2011sampling}. There have also been extensions considering perception uncertainty \cite{renganathan2020towards}. However, these algorithms only provide nominal paths, and assume that a separate low-level controller exists to generate collision-free trajectories at run time. For trajectory planning that takes into account non-trivial dynamical systems of the robot, kinodynamic \rrt{} \cite{lavalle2001randomized, lavalle2006planning} and closed-loop \rrt{} (\texttt{CL-RRT}, \cite{kuwata2008motion}) and \texttt{CL-RRT\#} grow the tree  by sampling control inputs and then propagating forward the nonlinear dynamics (with the optional use of stabilizing controllers and tree rewiring to approach optimality). Further in this line of work, there has been a relatively smaller amount of works on algorithms that focus on producing controllers as opposed to simple reference trajectories.
The \texttt{safeRRT} algorithm \cite{positiveInvariant,weiss2017motion} generates a closed-loop trajectory from an initial state to the desired goal by expanding a tree of local state-feedback controllers to maximize the volume of corresponding positive invariant sets while satisfying the input and output constraints. Based on the same idea and following the \rrt{} approach, the \texttt{LQR-tree} algorithm \cite{tedrake2009lqr} creates a tree by sampling over state space and stabilizes the tree with a linear quadratic regulator (LQR) feedback. With respect to the present paper, the common trait among all these works is the use of full state feedback (as opposed to output feedback).

Separately, the state-of-the-art method for synthesizing safe and stable control commands is represented by the combination of Control Barrier Functions (CBF) and Control Lyapunov Function (CLF) \cite{ames2014control,hsu2015control,borrmann2015control}; however, these approaches are in general not complete for complex environments (i.e., they might fail to reach the goal even when a feasible path exists).

In our algorithm, we use the min-max robust Linear Programming (LP) controller synthesis method from \cite{Mahroo}. However, that work assumes that a polyhedral convex cell decomposition of the environment is available, which greatly reduces the applicability of that method. Moreover, that work also does not test the resulting controllers in a real experimental setting.

\subsection{Proposed approach and contributions}
As mentioned above, at a high level, our approach first converts an implicit representation of the environment to a simplified tree graph via sampling, and then builds a sequence of linear output feedback controllers to generate piece-wise linear control laws for navigation. To build the tree, we use the \rrtstar{} algorithm with two modifications:
\begin{enumerate*}
\item we post-process the tree to minimize the number of nodes and decrease the overall path length of each branch, and
\item we do not discard the samples that are found to be in collision with obstacles.
\end{enumerate*}
In addition to the sample-based tree and collision samples, we assume that the environment includes a set of landmarks that the robot can sense, such that for any location in the free space at least one landmark is available (in practice, these landmarks could correspond to visual features on the surface of obstacles, although we do not place any restriction on their location).

We then propose a way to define convex cells around each node in the tree that ensure progress from that node to its parent via a CLF constraint, while using the samples found in collision to form a local convex approximation of the free space for obstacle avoidance via CBF constraints.
We apply the method of~\cite{Mahroo} to formulate a min-max robust Linear Program that synthesizes a controller for each cell which takes as inputs relative position measurements of the landmarks and outputs a control signal that respects and balances the stability constraint from the CLF, and the safety (collision avoidance) constraints from the CBF.~Additionally, we show how to easily recompute the controllers (online) to handle the case where subsets of landmarks are not visible (e.g., due to the the camera's limited field of view).
To summarize, the main contributions of this work are:
\begin{itemize}
\item Integrate high-level \rrtstar{} path planning, with low-level CLF-CBF LP control synthesis, thus allowing the method of \cite{Mahroo} to be applied to general environments for which a cell decomposition is not available;
\item Introduce a new algorithm to simplify and improve a tree generated by \rrtstar{} algorithm with a finite number of samples (which, in general, is not optimal).
\item Introduce a new method to  reformulate the controller to navigate with a limited field of view.
\item Obtain a method that, thanks to the use of visual features of the environment (landmarks) and output feedback controllers, automatically adapts \rrtstar{} solutions to deformations of the environment and deviations from the nominal path without replanning;
\item Implement the proposed algorithm in a real-world environment to validate the performance of our algorithm.
\end{itemize}


\section{BACKGROUND}\label{sec:background}
In this section, we review the CLF and CBF constraints, and the \rrtstar{} method in the context of our proposed work.
\subsection{System Dynamics}
We assume that the robot has control-affine dynamics of the form
\begin{equation}\label{sys1}
  \dot{x}=Ax+Bu,
\end{equation}
where $x \in \cX\subset\real{n}$ denotes the state, $u\in\cU\subset\real{m}{}$ is the system input, and $A\in\real{n\times n}$, $B\in\real{n\times m}$ define the linear dynamics of the system. We assume that the pair $(A,B)$ is controllable, and that $\cX_i$ and $\cU$ are polytopic,
\begin{align}\label{state_limits}
  \cX=\{x\mid A_{x}x\leq b_{x}\},&& \cU=\{u\mid A_{u}u\leq b_u\},
\end{align}
Note that, in our case, $\cX$ will be a convex cell centered around a sample in the tree (Section~\ref{sec:environment}).

\subsection{Tree graphs}
A graph is a tuple $(\cV,\cE)$ where $\cV$ represents a set of nodes and $\cE$ represents a set of edges. If $(i,j)\in \cE$, we say that $j$ is the parent of node $j$. An oriented tree $\cT$ is a graph where each node has exactly one parent, except for the \emph{root}, which has no parents. We refer to nodes without children as \emph{leaves}.
\subsection{Optimal Rapidly-Exploring Random Tree (\rrtstar{})}
\label{sec:rrtstar}
In this section we review \rrtstar{}, an algorithm which is typically used for single-query path planning, but that can also be used to build a representation of the free configuration space starting from a given root node (in this paper we use it for the latter purpose). The algorithm builds a tree $\cT$ and is summarized in Algorithm~\ref{alg:RRT*}, and its main functions are:
\begin{itemize}
\item \rrtfunc{RandomSample}: return a random sample from a uniform probability distribution in configuration space~$\cX$.
\item \rrtfunc{IsSampleCollision}: return \emph{True} if the given sample is in collision with an obstacle.
\item \rrtfunc{Nearest}: return the node in $\cV$ closest to the random sample $x_{\text{rand}}$.
\item \rrtfunc{Near}: return a set of nodes in $\cV$ within distance $r^*$ from the $x_{\text{rand}}$ where $r^*$ is defined as
  \begin{equation}
    r^* = \min (\gamma^*(\frac{\log{\abs{\cV}}}{\cV})^{\frac{1}{d+1}},\eta),
  \end{equation}
  $d$ is the dimension of the configuration space, $\eta$ is the constant in the definition of the \rrtfunc{Steering} function, $\gamma^* = 2((1+\frac{1}{d})\frac{A_{\text{free}}}{\pi})^{\frac{1}{d}}$, and $A_{\text{free}}$ represents the area of the free space.
\item \rrtfunc{Steering}: given two states $p$ and $p'$, and $\eta$, return a path from $p$ toward $p'$ with length $\eta$ if there is no collision between the path and the obstacles.
\item \rrtfunc{isEdgeCollision}: given two states $p$ and $p'$, return \emph{True} if there is no collision between the path that connects $p$ to $p'$ and the obstacles; note that in general this function is typically built by using \rrtfunc{IsSampleCollision}.
\item \rrtfunc{Rewire:} check if the cost of the nodes in $X_{\text{near}}$ is less through $x_{\text{new}}$ as compared to their older costs, then its parent is changed to $x_{\text{new}}$.
\end{itemize}
For this paper, the only modification to the original \rrtstar{} algorithm is represented by line~$\ref{al:store collision}$ in Algorithm~\ref{alg:RRT*}, which stores random samples that were found to be in collision with an obstacle in the list $\cV_{\text{collision}}$ instead of discarding them; this list is then returned by the algorithm and will be used to define the CBF constraints in our algorithm (see Section~\ref{sec:tree-CBF}).
In general, \rrtstar{} is guaranteed to be asymptotically complete and optimal, although these guarantees do not necessarily hold with a finite number of samples.

\begin{algorithm}[t]
  \scriptsize
  \caption{\rrtstar{}}
  \begin{algorithmic}[1]
    \State Input (Obstacle lists $\cO$, start point, max$_{\text{itr}}$, $\eta$)
    \State $\cV\gets$ start point, $\cE\gets \emptyset$, $\cV_{\text{collision}}\gets\emptyset$
    \For {$i=1,\hdots, \text{max}_{\text{itr}}$}
    \State $x_{\text{rand}}\leftarrow$ \rrtfunc{RandomSample}
    \If {\rrtfunc{IsSampleCollision($x_{\text{rand}}$)}}
    \label{al:store collision}\State Append $x_{\text{rand}}$ to $\cV_{\text{collision}}$ and break
    \EndIf
    \State   $x_{\text{nearest}}\leftarrow$ \rrtfunc{Nearest}$(G=(V,E),x_{\text{rand}})$
    \State   $x_{\text{new}}\leftarrow$ \rrtfunc{Steer}$(x_{\text{nearest}},x_{\text{rand}},\eta)$
    \If {\rrtfunc{isEdgeCollision}$(x_{\text{nearest}}, x_{\text{new}})$}
    \State  $X_{\text{near}}\leftarrow$ \rrtfunc{Near}$(G=(V,E),x_{\text{new}},r^{*})$
    \State $\cV\leftarrow \cV \cup {x_{\text{new}}}$
    \State $x_{\text{min}} \leftarrow x_{\text{nearest}}$
    \State $c_{\text{min}} \leftarrow$ cost$_{\text{nearest}}+\text{norm}(x_{\text{nearest}},x_{\text{new}})$
    \For{every $x_{\text{near}}\in X_{\text{near}}$}
    \If {\rrtfunc{isEdgeCollision} $\wedge$ cost$_{\text{near}}+\text{norm}(x_{\text{near}},x_{\text{new}})<c_{\text{min}}$}
    \State  $x_{\text{min}}\leftarrow x_{\text{near}}$
    \State  $c_{\text{min}} \leftarrow$ cost$_{\text{near}}+\text{norm}(x_{\text{near}},x_{\text{new}})$
    \EndIf
    \EndFor
    \State  $\cE\leftarrow \cE \cup{(x_{\text{min}},x_{\text{new}})}$
    \State $\cE\leftarrow$\rrtfunc{Rewire}$((\cV,\cE),X_{\text{near}},x_{\text{new}})$
    \EndIf
    \EndFor
    \State return $\cT =(\cV, \cE)$, $\cV_{\text{collision}}$.
  \end{algorithmic}
  \label{alg:RRT*}
\end{algorithm}

\begin{algorithm}[t]
  \scriptsize
  \caption{Rewire}
  \begin{algorithmic}[1]
    \For{every $x_{\text{near}}\in X_{\text{near}}$}
    \If {\rrtfunc{isEdgeCollision} $\wedge$ cost$_{\text{near}}+\text{norm}(x_{\text{near}},x_{\text{new}})<c_{\text{min}}$}
    \State  $x_{\text{parent}}\leftarrow$ \rrtfunc{Parent}($x_{\text{near}}$)
    \State  $\cE\leftarrow (\cE-\{x_{\text{parent}},x_{\text{near}}\}) \cup{(x_{\text{new}},x_{\text{near}})}$
    \EndIf
    \EndFor
    \State return $\cE$.
  \end{algorithmic}
  \label{alg:rewire}
\end{algorithm}

\subsection{Control Lyapunov and Barrier Functions (CLF, CBF)}\label{sec:ECBF}
In this section we review CLF and CBF constraints, which are differential inequalities that ensure stability and safety (set invariance) of a control signal $u$ with respect to the dynamics \eqref{sys1}. These constraints are defined. First, it is necessary to review the following.

\begin{definition}
  The Lie derivative of a differentiable function $h$ for the dynamics \eqref{sys1} with respect to the vector field $Ax$ is defined as $\cL_{Ax}h(x)=\frac{\partial h(x(t))}{\partial x}\transpose Ax$.
\end{definition}
Applying this definition to \eqref{sys1} we obtain
\begin{equation}\label{Lie}
  \dot{h}(x)=\cL_{Ax}h(x)+\cL_Bh(x)u.
\end{equation}

In this work, we assume that Lie derivatives of $h(x)$ of the first order are sufficient \cite{Isidori:book95} (i.e., $h(x)$ has relative degree 1 with respect to the dynamics \eqref{sys1}); however, the result could be extended to the higher relative degree, as discussed in \cite{Mahroo}.

We now pass to the definition of the differential constraints.
Consider a continuously differentiable function $V(x):\cX\to\real{}$, $V(x)\geq 0$ for all $x\in\cX$, with $V(x)=0$ for some $x\in\cX$.
\begin{definition}\label{def:ECLF}
  The function $V(x)$ is a \textit{Control Lyapunov Function} (CLF) with respect to \eqref{sys1} if there exists positive constants $c_1,c_2,c_v$ and control inputs $u\in \cU$ such that
  \begin{equation}\label{cons:clf}
    \begin{aligned}
      \cL_AV(x)+\cL_BV(x)u+c_v V(x)\leq 0,\forall x \in \cX.
    \end{aligned}
  \end{equation}
  Furthermore, \eqref{cons:clf} implies that $\lim_{t\to\infty}V(x(t))=0$.
\end{definition}
Consider a continuously differentiable function $h(x):\cX\to\real{}$ which defines a safe set $\cC$ such that
\begin{equation}\label{set_c}
  \begin{aligned}
    \cC&=\{x\in \real{n}|\;h(x)\geq0\},\\
    \partial \cC&=\{x\in \real{n}|\;h(x)=0\},\\
    {Int}(\cC)&=\{x\in \real{n}|\;h(x)>0\}.
  \end{aligned}
\end{equation}
In our setting, the set $\cC$ will represent a convex local approximation of the free configuration space (in the sense that $x\in\cC$ does not contain any sample that was found to be in collision).
We say that the set $\cC$ is \emph{forward invariant} (also said \emph{positive invariant} \cite{positiveInvariant}) if  $x(t_0) \in \cC$ implies $x(t)\in \cC$, for all $t\geq 0$~\cite{zcbf1}.
\begin{definition}[CBF, \cite{nguyen2016exponential}]\label{def:ECBF}
  The function $h(x)$ is a \emph{Control Barrier Function} with respect to \eqref{sys1} if there exists a positive constant $c_h$, control inputs $u\in \cU$, and a set $\cC$ such that
  \begin{equation}\label{cons:cbf}
    \cL_{Ax}h(x)+\cL_Bh(x)u+c_hh(x)\geq 0,\forall x \in \cC.
  \end{equation}
  Furthermore, \eqref{cons:cbf} implies that the set $\cC$ is forward invariant.
\end{definition}


\subsection{Environment}\label{sec:environment}
As mentioned in the previous section, the environment is implicitly represented by the \rrtfunc{Sample} function.
Additionally, we assume that the robot can measure the displacement $\hat{l}_i-x$ between its position $x$ and its set of \emph{landmarks} $\hat{l}_i$. The location of $\hat{l}_i$ is assumed to be known and fixed in the environment.
Note that the landmarks $\hat{l}_i$, from the point of view of our algorithms, can be arbitrary as long as there is at least one landmark visible from any point $x$ in the free configuration space. The landmarks do not need to be chosen from the samples of the \rrtstar{} algorithm, or from the obstacles. Furthermore, we assume that the total extent of the environment is bounded by a convex polyhedron $\cX_{\text{env}}$ (e.g., simple box constraints).
\section{FEEDBACK CONTROL PLANNING VIA \rrtstar}\label{sec:problem-setup}
At a high level, our algorithm first divides the configuration space into cells according to a tree-graph representation of the environment, and then computes a controller for each cell that can be used to move the robot along the tree starting from any initial location. More in detail, our solution is comprised of the following steps:
\begin{enumerate}
\item Run \rrtstar, and then simplify the generated tree.
\item Define a convex cell around every node of the tree while taking into account the position of its parent.
\item Define the CLF and CBF constraints for each cell.
\item Use a robust LP formulation to compute a controller for each cell that respects the CLF and CBF constraints.
\item Reformulate the controller in terms of visible and invisible landmarks for the limited field of view of the robot.
\end{enumerate}
Below we give the details of each one of the steps.
\begin{figure}[t]
  \centering
  \subfloat[Tree generated by \rrtstar{}]{\label{fig:tree-1}{\includegraphics[width=4cm]{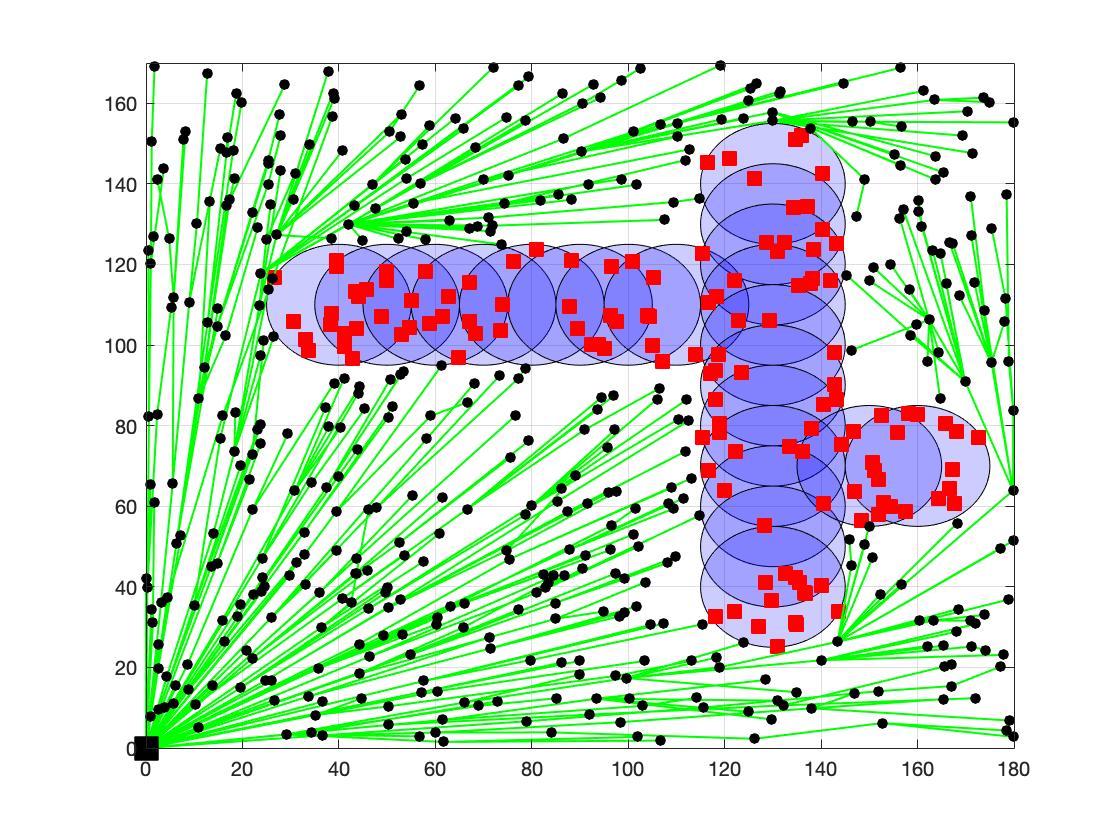}}}
  \subfloat[Simplified tree]{{\label{fig:tree-2}\includegraphics[width=4cm]{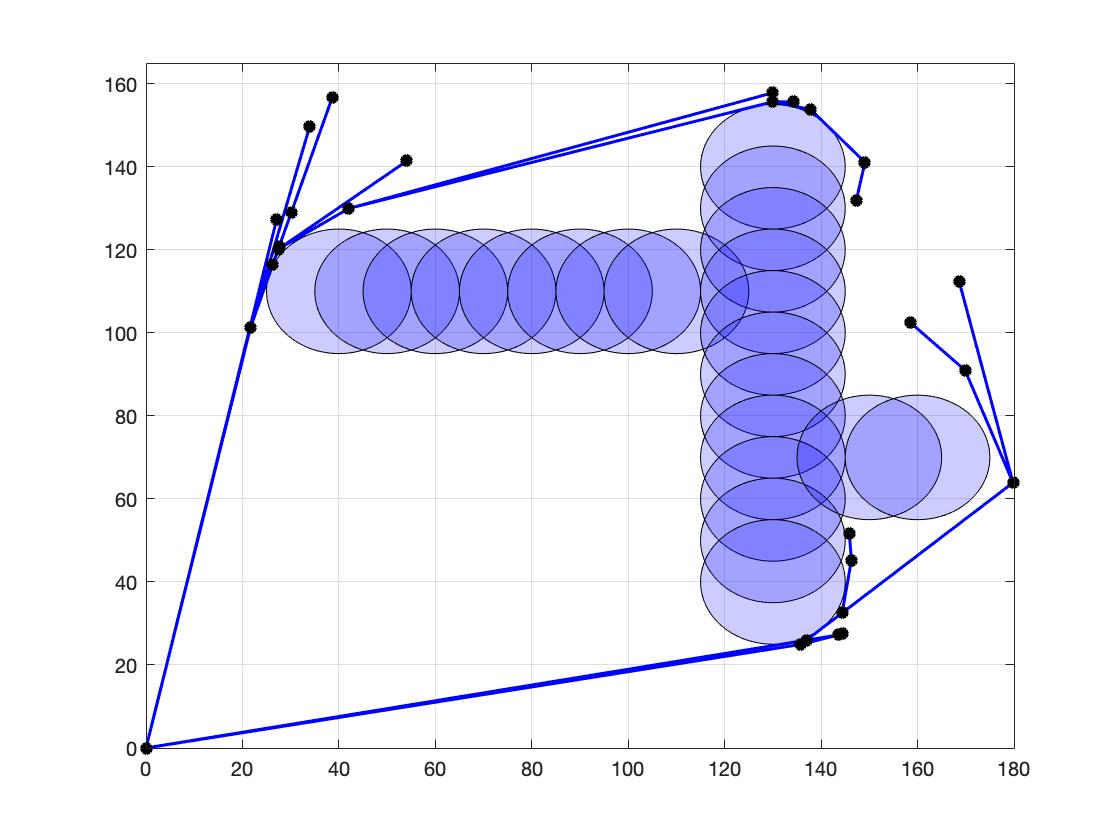}}}%

  \caption{Obstacles are represented by blue circles, and the start point of \rrtstar{} is located at the origin. In Fig.~\ref{fig:tree-1}, yellow dots show the samples in collision with the obstacles, and the generated tree from \rrtstar{} is plotted in green. Fig.~\ref{fig:tree-2} depicts the simplified tree following Section~\ref{sec:simplified-tree}.}
  \label{fig:tree}
\end{figure}
\subsection{Simplified Tree Graph}\label{sec:simplified-tree}
We start with a tree $\cT = (\cV,\cE)$ generated by the traditional \rrtstar{} algorithm from Section~\ref{sec:rrtstar}. Since the number of samples is finite, the generated tree is not asymptotically optimal but it has a large number of nodes. We simplify the tree such that the tree has less number of nodes while it keeps track of all samples in collision with obstacles
by following three steps:
\begin{enumerate}
\item\label{it:ppr} \emph{Post Processing Rewiring} (PPR, Algorithm~\ref{alg:rewiring}): similarly to $\mathtt{\theta^*}$ \cite{nash2007theta}, we examine each node starting from the root and using a breadth-first order, and use the function \rrtfunc{isEdgeCollision} to check if it can be connected to an ancestor (testing from the parent and then moving toward the root) without collisions and while lowering the path length.
\item Remove Crossing (RC, Algorithm~\ref{alg:rc}), if edge $(i,j)$ crosses edge $(p,q)$ with an intersection at point $k$,
  we add point $k$ to $\cV$, and edges $(i,k)$ and $(p,k)$ to $\cE$. Then, we compare the costs of reaching the start point from $k$ through edges $(k,j)$ and $(k,q)$, and add the smallest one to $\cE$ as the parent of node $k$. 
\item\label{it:ctl} Cutting the Leaves (CtL, Algorithm~\ref{alg:cutting}), for a node that has multiple leaves as children, we only keep a single leaf in the middle. 
\item We repeat steps \ref{it:ppr}-\ref{it:ctl} until there are no changes in $\cT$.

\end{enumerate}
Fig.~\ref{fig:tree} shows an example of the procedure, starting from the \rrtstar{} tree (Fig.~\ref{fig:tree-1}), and ending with the simplified tree after the three steps (Fig.~\ref{fig:tree-2}).

Note that as a consequence of the simplifying steps above, it is possible to connect each sample from the original \rrtstar{} to the simplified tree with a straight line, which suggests that the simplified tree will be a good roadmap representation \cite{Choset:book05} of the free configuration space reachable from the root.

\begin{algorithm}[t]
  \scriptsize
  \caption{Post Processing Rewiring}
  \label{alg:rewiring}
  \begin{algorithmic}[1]
    \State Input ($\cT=(\cV,\cE)$)
    \For {$i=1,\hdots,\abs{V}$}
    \While {$\exists$ \rrtfunc{Parent}(node) $\wedge$ $\exists$ \rrtfunc{Parent(Parent}(node)) $\wedge$ \rrtfunc{IsEdgeCollision}(node$_i$,\rrtfunc{Parent(Parent}(node$_i$))}
    \State $\cE\leftarrow \cE-$\{node$_i$,\rrtfunc{Parent}(node$_i$)\} $\cup  $ {node$_i$,\rrtfunc{Parent(Parent}(node$_i$))}
    \EndWhile
    \EndFor
    \State return $\cT=(\cV,\cE)$
  \end{algorithmic}
\end{algorithm}

\begin{algorithm}[t]
  \scriptsize
  \caption{Remove Crossing}
  \label{alg:cutting}
  \begin{algorithmic}[1]
    \State Input ($\cT=(\cV,\cE)$)
    \For {for every $(i,j) \in \cE$}
    \For{for every $(p,q) \in \cE$}
    \If{edge $(i,j)$ and $(p,q)$ have an intersection}
    \State $k = $intersection $(i,j)$ and $(p,q)$
    \State  \rrtfunc{Parent}$(i)\leftarrow k$
    \State  \rrtfunc{Parent}$(p)\leftarrow k$
    \State $\cE \leftarrow \cE-\{(i,j),(p,q)\}$
    \State $\cE \leftarrow \cE\cup{\{(i,k),(p,k)\}}$
    \State $\text{cost}_j \leftarrow$ reach start point through node $j$
    \State $\text{cost}_q \leftarrow$ reach start point through node $q$
    \If{$\text{cost}_j \leq \text{cost}_q$}
    \State $\cE \leftarrow \cE\cup{\{(k,j)\}}$
    \EndIf
    \If{$\text{cost}_q \leq \text{cost}_j$}
    \State $\cE \leftarrow \cE\cup{\{(k,q)\}}$
    \EndIf
    \EndIf
    \EndFor
    \EndFor
    \State return  $\cT=(\cV,\cE)$
  \end{algorithmic}
  \label{alg:rc}
\end{algorithm}

\begin{algorithm}[t]
  \scriptsize
  \caption{Cutting the Leaves}
  \label{alg:cutting}
  \begin{algorithmic}[1]
    \State Input ($T=(V,E)$)
    \For {$i=1,\hdots,\abs{V}$}
    \For{leaves $j$ connected to $i$}
    \State angles$_i$ $\leftarrow$ angles$_i$ $\cup$ \rrtfunc{Angle} =$(i,j)$
    \EndFor
    \State $j_{largest}$ $\leftarrow$ \rrtfunc{LargestAngle}(angles$_i$)
    \State $j_{smallest}$ $\leftarrow$ \rrtfunc{SmallestAngle}(angles$_i$)
    \State $E_s \leftarrow E_s \cup \{i,j_{largest}\}\cup \{i,j_{smallest}\}$
    \State $V_s \leftarrow V_s \cup \{j_{largest}\}\cup \{j_{smallest}\}$
    \EndFor
    \State return $T_s=(V_s,E_s)$
  \end{algorithmic}
  \label{alg:ppr}
\end{algorithm}
\subsection{Environment Constraints}\label{tree-X}
For each edge $\ijE$ in the tree, we define a cell $\cX_{ij}$~as
\begin{equation}
  \cX_{ij} = \cX_{V_i} \cap \cX_{P_i} \cap \cX_{\text{env}}, \label{Xi}
\end{equation}
where 
\begin{subequations}\label{state_convex}
  \begin{align}
    &
   \small{\cX_{V_{ij}} =\{x| (x-x_i)\transpose\frac{x_k-x_i}{\norm{x_k-x_i}}\leq\frac{\norm{x_k-x_i}}{2}, k\in\cV\setminus\{j\}\} \label{Xvi}} \\
    & \small {\cX_{P_{ij}} =\{x| (x-x_i)\transpose\frac{x_j-x_i}{\norm{x_j-x_i}}\leq \norm{x_j-x_i}, i,j\in\cV\}},
      \label{Xij}
 \end{align}
\end{subequations}
and $\norm{x_k-x_i}$ and $x \in\real{n}$ is the Euclidean distance between nodes $i,k\in\cV$; the polyhedron $\cX_{ij}$ is similar to a Voronoi region \cite{latombe2012robot}, and is defined by a set of perpendicular bisector of segments $i,k$ for $k\in \cV\setminus\{j\}$, and by the line perpendicular to $i,j$ passing through $j$. The inequalities in \eqref{Xvi}-\eqref{Xij} can be written in matrix form of \eqref{state_limits}. Note that $\cX_{ij}$ contains all the points that are closest to $i$ than other vertices in $\cT$, but it also includes the parent $j$; we empirically noticed that with the latter modification we obtained more robust results. An example of $\cX_{ij}$ is shown in Fig.~\ref{fig:CBF-X}
\subsection{Stability by CLF}\label{sec:tree-CLF}
To stabilize the navigation along an edges of a tree, we define the Lyapunov function $V_{ij}(x)$ as
\begin{equation}\label{V}
  V_{ij}(x)=z_{ij}^T (x-x_{j}),
\end{equation}
where $z_{ij}\in \real n$ is the  \text{\emph{exit direction}} for edge $(i,j)$, $x_{j} \in \text{\emph{exit face}}$ is the position of the parent of node $i$, and $V_{ij}(x)$ reaches its minimum $V(x)=0$ at $x_j$. Note that the Lyapunov function represents, up to a constant, the distance $d(x,x_j)$ between the current system position and the exit face. By Definition~\ref{def:ECLF}, $V_{ij}(x)$ is a CLF.
\begin{definition}\label{exitface}
 For $\cX_{ij}$, we define the exit direction as $z_{ij}= \frac{x_j-x_i}{\norm{x_j-x_i}}$ that is a unit vector from node $j$ towards node $i$ where $j\in\cV_s$ is the parent of node $i$.
\end{definition}

\subsection{Safety by CBF}\label{sec:tree-CBF}
In this section, we define barrier functions $h_{ij}(x)$ that defines a cone representing a local convex approximation of the free space between $i$ and $j$, in the sense that it excludes all samples in $\cV$ that are on the way from $i$ to $j$. In particular, we use the following steps (we consider only the 2-D case, although similar ideas could be applied to the 3-D case):
\begin{enumerate}
\item Define set $\cO_{ij}\subset\cV_{\text{collision}}$ whose projection falls on the segment $i,j$, i.e., 
  \begin{equation}
    \cO_{ij}=\{o\in\cV_\text{collision}\mid0\leq\proj_{ij}(o)\leq 1\}
  \end{equation}
  where
  \begin{equation}
    \proj_{ij}(o)=\frac{(\vct{x}_i-\vct{x}_j)\transpose(\vct{x}_i-\vct{x}_o)}{\norm{\vct{x}_i-\vct{x}_j}}
  \end{equation}
  is the scalar projection of the vector $o$ onto the segment $i,j$.
\item From the set $\cO_{ij}$, we choose two samples such that
  \begin{equation}\label{o_cbf}
    \begin{aligned}
      & o_{u} =\min_{o\in\cO_{ij}}\{\vct{x}_o-\text{proj}^{io}_{ij} | \theta_{io}>0\} \\
      &  o_{d} = \min_{o\in\cO_{ij}}\{\vct{x}_o-\text{proj}^{io}_{ij} | \theta_{io}<0\}
    \end{aligned}
  \end{equation}
  where $\theta_{io}=\angle(i,j,o)$ is the oriented angle between edge $(i,j)$ and line $(x_o,x_i)$.
\item We write the equations of two lines passing through $\{i,o_{u}\}$ and $\{i,o_{d}\}$ in a matrix form using $A_{h_{i}}\in \real{2\times n}$, $b_h\in \real{2}$ to define the invariant set
  \begin{equation}
    \cC_{ij}=\{x\in\real{2}:A_{h_{ij}}\vct{x}+b_{h_{i}}>0\}.
  \end{equation}
\end{enumerate}
The corresponding CBF is then defined as
\begin{equation}\label{h}
  h_{ij}(x)=A_{h_{ij}}x+b_{h_{ij}}.
\end{equation}
An example of the set $\cC_{ij}$  is shown in Fig.~\ref{fig:CBF-X}. Note that the region $\cC_{ij}$ might not include the entire cell $\cX_{ij}$. However, the controller will be designed to satisfy the CBF and CLF constraints over the entire cell $\cX_{ij}$; in practice, this means that if the robot starts in the region $\cX_{ij}\setminus\cC_{ij}$, it will be driven toward the boundary of $\cC_{ij}$ (this is a consequence of the CBF constraint, and can be proved in a similar way as the original result \cite{ames2014control}).
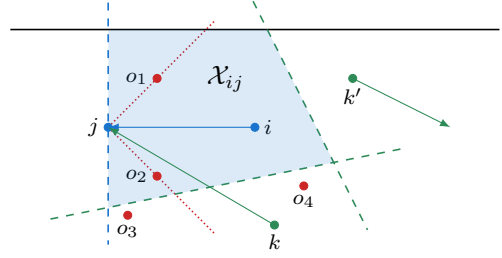
\begin{figure}[t]
  \centering
  \begin{tikzpicture}[scale=1.3]
  \tkzDefPoints{0/0/j, 1.5/0/i}

  \tkzDefPoints{0.5/0.5/o_1,0.5/-0.5/o_2,0.2/-0.9/o_3,2/-0.6/o_4}

  \tkzDefPoints{1.7/-1/k,2.5/0.5/k',3.5/0/k'p}

  \tkzLabelPoints[font=\footnotesize,below](o_4,o_3,k,k')
  \tkzLabelPoints[font=\footnotesize,left](o_1,o_2,j)
  \tkzLabelPoints[font=\footnotesize,right](i)

  \tkzDefPoints{-1/1/b_1,4/1/b_2}
  \tkzDrawSegment[black](b_1,b_2)

  \tkzSetUpLine[style=dashed]
  \tkzDefLine[mediator](i,k')
  \tkzGetPoints{ik'u}{ik'd}
  \tkzDefPointWith(ik'u,ik'd)
  \tkzDrawSegment[SeaGreen4,add=0.1 and 0.25](ik'u,ik'd)

  \tkzDefLine[mediator](i,k)
  \tkzGetPoints{iku}{ikd}
  \tkzDefPointWith(iku,ikd)
  \tkzDrawSegment[SeaGreen4,add=0.3 and 0.8](iku,ikd)

  \tkzDefLine[perpendicular=through j](j,i)
  \tkzGetPoint{jiu}
  \tkzDrawSegments[DodgerBlue3,add=-0.15 and 0.8](jiu,j)

  \tkzSetUpLine[color=Firebrick3,style=densely dotted,add=0 and 1.2]
  \tkzDrawLine(j,o_1)
  \tkzDrawLine(j,o_2)

  \tkzInterLL(jiu,j)(b_1,b_2)
  \tkzGetPoint{jib_1b_2}

  \tkzInterLL(b_1,b_2)(ik'u,ik'd)
  \tkzGetPoint{b_1b_2ik'}

  \tkzInterLL(ik'u,ik'd)(iku,ikd)
  \tkzGetPoint{jikk'}

  \tkzInterLL(jiu,j)(iku,ikd)
  \tkzGetPoint{jik}

  \fill[DodgerBlue3,opacity=0.15,very thick] (jib_1b_2) -- (b_1b_2ik') -- (jikk') -- (jik) -- cycle;

  \tkzDrawPoints[Firebrick3](o_1,o_2,o_3,o_4)
  \tkzDrawPoints[SeaGreen4](k,k')
  \tkzDrawPoints[DodgerBlue3](j,i)
  \draw[-latex,DodgerBlue3] (i) -- (j);
  \begin{scope}[SeaGreen4,-latex]
    \draw (k) -- (j);
    \draw (k') -- (k'p);
  \end{scope}

  \node at (1.2,0.5) {$\cX_{ij}$};
\end{tikzpicture}

  \caption{The Voronoi-like region $\cX_{ij}$ for edge $(i,j)$, and the corresponding CBF constraints; green points and arrows: other vertices and edges in the tree $\cT$; red points: collision samples in $\cV_s$; black line: $\cX_{\text{env}}$; green dashed lines: $\cX_{V_{ij}}$ and $\cX_{P_{ij}}$; shaded region: $\cX_{ij}$; red dotted lines: CBF constraints $h_i$. For the CBF constraints, note that the projection of vector $(i,o_4)$ onto edge $(i,j)$ does not lie  between node $i$ and node $j$, hence $o_4$ does not support a constraint; moreover, from \eqref{o_cbf}, $o_d = o_1$ and $o_u= o_2$ are the closet points (in angle) to the edge $(i,j)$ and are on two sides of the edge $(i,j)$.
}\label{fig:CBF-X}
\end{figure}
\subsection{Controller}\label{sec:tree-controller}

\newcommand{\xpos}{x_{\textrm{pos}}}
We assume that the robot can only measure the relative displacements between the robot's position $x$ and the landmarks in the environment, which corresponds to the output function
\begin{equation}\label{vec-landmarks}
  y=(L-x\vct{1}\transpose)^\vee=L^\vee-\cI x=\stack{(l_i- x)},
\end{equation}
where $L_\in\real{n\times n_l}$ is a matrix of landmark locations, $i=1,\hdots,n_l$ that $n_l$ is the number of landmarks,  $A^\vee$ represents the vectorized version of a matrix $A$, $\cI=\vct{1}_{nl} \otimes I_n$, and $\otimes$ is the Kronecker product. Our goal is to find a feedback controller of the form
\begin{equation}\label{u}
  u_{ij}(K)=K_{ij}y,
\end{equation}
where $K_{ij} \in \real{m\times nn_l}$ are the feedback gains that need to be found for each cell $\cX_{ij}$. Intuitively, a controller of the form \eqref{u} corresponds to a control command that is a weighted linear combination of the measured displacements $y$. The goal is to design $u(y)$ such that the system is driven toward the exit direction $z_{ij}$ while avoiding obstacles. Note that, to define a controller for edge $(i,j)$, the landmarks do not necessarily need to belong to $\cX_{ij}$, and, in general, each cell could use a different set of landmarks (we explore this direction further in Section~\ref{sec:limited-field-of-view}). 

Following the approach of \cite{Mahroo}, and using the CLF-CBF constraints reviewed in Section~\ref{sec:background}, we encode our goal in the following feasibility problem:

\begin{equation}\label{findK}
  \begin{aligned}
    & \textrm{find} \;\;{K_{ij}}\\
    & \textrm{subject to}:\\
    &\{\text{CBF:}\;-( \cL_{A_ix}h_{ij}(x)+\cL_Bh_{ij}(x)u+c\transpose_hh_{ij}(x))\}\leq 0,\\
    & \{\text{CLF:}\quad\cL_{A_ix}V_{ij}(x)+\cL_BV_{ij}(x)u+c\transpose_v{V_{ij}}(x)\}\leq 0,\\
    &u\in\cU,\;\;\forall x\in \cX_{ij},\;\;(i,j)\in\cE.
  \end{aligned}
\end{equation}

In practice, we aim to find a controller that satisfies the constraints in \eqref{findK} with some margin. Note that the constraints in \eqref{findK} need to be satisfied for all $x$ in the region $\cX_{ij}$, i.e., the same control gains should satisfy the CLF-CBF constraints at every point in the region. We handle this type of constraint by rewriting \eqref{findK} using a min-max formulation~\cite{Mahroo}, arriving to the following robust optimization problem:
\begin{equation}\label{opt_margin}
\centering
  \begin{aligned}
    &\min_{K_{ij},S_{V_{ij}},S_{h_{ij}}}\;w_{h_{ij}}\transpose S_{h_{ij}}+w_{V_{ij}}S_{V_{ij}}\\
    &\textrm{\subjectto}\\
    &\max_x \text{CBF}\leq S_{h_{ij}} ,\\
    &\max_x \text{CLF}\leq S_{V_{ij}},\\
    & S_{V_{ij}},S_{h_{ij}} \leq 0,\;\;u\in\cU,\;\; \forall x\in \cX_{ij},\;\; (i,j)\in\cE.
  \end{aligned}
\end{equation}
From \eqref{Lie}, the Lie derivatives of $h_{ij}(x)$ and $V_{ij}(x)$ are written as:
\begin{equation}\label{cons-Lie}
  \begin{aligned}
    & \dot{h}_{ij}(x)=A_{hij}\dot{x} =A_{hij}(A x+B u),\\
    & \dot{V}_{ij}(x) =z_{ij}\transpose\dot{x}= z_{ij}\transpose (Ax+Bu).
  \end{aligned}
\end{equation}
Combining \eqref{vec-landmarks}, \eqref{u}, and \eqref{cons-Lie} with \eqref{sys1}, the constraints in \eqref{opt_margin} can be rewritten as:
\begin{equation}\label{primal-cbf}
  \begin{aligned}
    & \text{CBF constraint:}\\
    &\begin{bmatrix}
      \underset{x}{\max}-(A_{hij}A-A_{hij}B K_{ij}\cI+{c_h}A_{hij})x\\
      \subjectto\;\;A_{xij} x \leq b_{xij}\\
    \end{bmatrix}
    \leq \\&
    \quad \quad  \quad \quad \quad \quad \quad  \quad \quad  \quad \quad
    S_{h_{ij}}+c_bb_{hij}+{A_{hij}}B K_{ij}L_{ij}^\vee,\\
  \end{aligned}
\end{equation}
\begin{equation}\label{primal-clf}
  \begin{aligned}
    & \text{CLF constraint:}\\
    &\begin{bmatrix}\underset{x}{\max}(z_{ij}\transpose A-z_{ij}\transpose BK_{ij}\cI+{c_v}z_{ij}\transpose )x \\
      \subjectto\;\;A_{xij} x \leq b_{xij}

    \end{bmatrix}
    \leq \\&
    \quad \quad  \quad \quad \quad \quad \quad  \quad \quad  \quad \quad S_{Vij}+c_vz_{ij}\transpose x_{j}-z_{ij}^T B K_{ij}L_{ij}^\vee,\\
  \end{aligned}
\end{equation}
Constraints in \eqref{primal-cbf}, \eqref{primal-clf} are linear in terms of variable $x$, so we can write dual forms of the constraints as
\begin{equation}\label{dual-conscbf}
  \begin{aligned}
    & \text{CBF dual constraint:}\\
    &\begin{bmatrix}
      &\min_{\lambda_{bj}} \lambda_{bij}\transpose b_{xij} \\
      &\subjectto\\
      & A_{xij} \transpose\lambda_{bij}=(-A_{hij}A+A_{hij}B K_{ij}\cI-{c_h}A_{hij})\transpose\\
      & \lambda_{bij}\geq 0,
    \end{bmatrix} \leq\\& \quad  \quad \quad  \quad \quad  \quad \quad  \quad \quad \quad
    S_{hij}+c_hb_{hij}+{A_{hij}}B_{ij} K_{ij}L_{ij}^\vee,
  \end{aligned}
\end{equation}
\begin{equation}\label{dual-consclf}
  \begin{aligned}
    & \text{CLF dual constraint:}\\
    &\begin{bmatrix}
      &\min_{\lambda_l}\lambda_{lij}  \transpose b_{xij}\\
      &\subjectto\\
      &  A_{xij}\transpose \lambda_{lij}
      =(z_{ij}\transpose A-z_{ij}\transpose BK_{ij}\cI+{c_v}z_{ij}\transpose )\transpose\\
      & \lambda_{lij} \geq 0
    \end{bmatrix}\leq \\& \quad \quad  \quad \quad  \quad \quad  \quad \quad  \quad \quad \quad S_{Vij}+c_vz_{ij}\transpose x_{j}-z_{ij}^TB K_{ij}L_{ij}^\vee,
  \end{aligned}
\end{equation}
Consequently, \eqref{opt_margin} with the dual constraints becomes:
\begin{equation}\label{opt-dual}
  \begin{aligned}
    &\min_{K,S_{V},S_{h}}\;w_{h}\transpose S_{h}+w_{V}S_{V}\\
    &\textrm{subject to}:\\
    &\text{CBF dual constraint} ,\\
    &\text{CLF dual constraint},\\
    & S_h,S_V \leq 0,
  \end{aligned}
\end{equation}
From \cite[Lemma~1]{Mahroo}, the optimization problem \eqref{opt_margin} is equivalent to \eqref{opt-dual}. 

Staring from a point $x\in\cX_{ij}$, $u_{ij}$ drives the robot toward $x_j$, the robot switches its controller to  $u_{jq}$ when $\norm{x-x_j}\leq\epsilon$ where node $q$ is the parent of node $j$. 
\subsection{Control With the Limited Field of View}\label{sec:limited-field-of-view}
In the formulation above and in the work of \cite{Mahroo}, it is implicitly assumed that the controller has access to all the landmarks measurements at all times. However, in practice, a robot will only be able to detect a subset of the landmarks due to a limited field of view or environment occlusions. To tackle this issue, we show in this section that the controller $u$ \eqref{u} can be designed using multiple landmarks (as in the preceding section), but then computed using a single landmark.
\begin{proposition} \label{prop:limited field view}
  Let $K=\bmat{K_1,\cdots,K_i,\cdots, K_l}$ be a partition of the controller matrix conformal with $L^\vee$. Given an arbitrary landmark $\hat{l}_i$ (column of $L$), the controller \eqref{u} can be equivalently written as
  \begin{equation}\label{u_new}
    u = \sum_jK_jy_i+k_{\text{bias},i}
  \end{equation}
  where $k_{\text{bias},i}\in\real{n}$ is a constant vector given by
  \begin{equation}
    k_{\text{bias},i}=\sum_{j\neq i}K_j(\hat{l}_j-\hat{l}_i)
  \end{equation}
\end{proposition}
\begin{proof}
  Using the conformal partition of $K$, we can expand \eqref{u} as
  \begin{equation}\label{decomU}
    u = \sum_jK_j(\hat{l}_j-x)
  \end{equation}
  Adding and subtracting $\sum_j K_j y_i$ and reordering, we have
  \begin{equation}\label{u_limited}
    u = 
    \sum_j K_j (\hat{l}_i-x)+\sum_j K_j(\hat{l}_i-\hat{l}_j),
  \end{equation}
  from which the claim follows.
\end{proof}
Using the fact that the global positions of the landmarks are known during planning, our new Proposition~\ref{prop:limited field view} shows that it is possible to implement the controller $u$ by measuring a single displacement $y_i$; moreover, since the original controller \eqref{u} is smooth, one can also switch among different landmarks without introducing discontinuities in the control. Although we stated our result for a single landmark, it is possible to prove a similar claim for any subset of landmarks.
\begin{figure}[b]
  \centering
  \subfloat[Original environment]{\label{fig:Matlab-org}{\includegraphics[width=4cm]{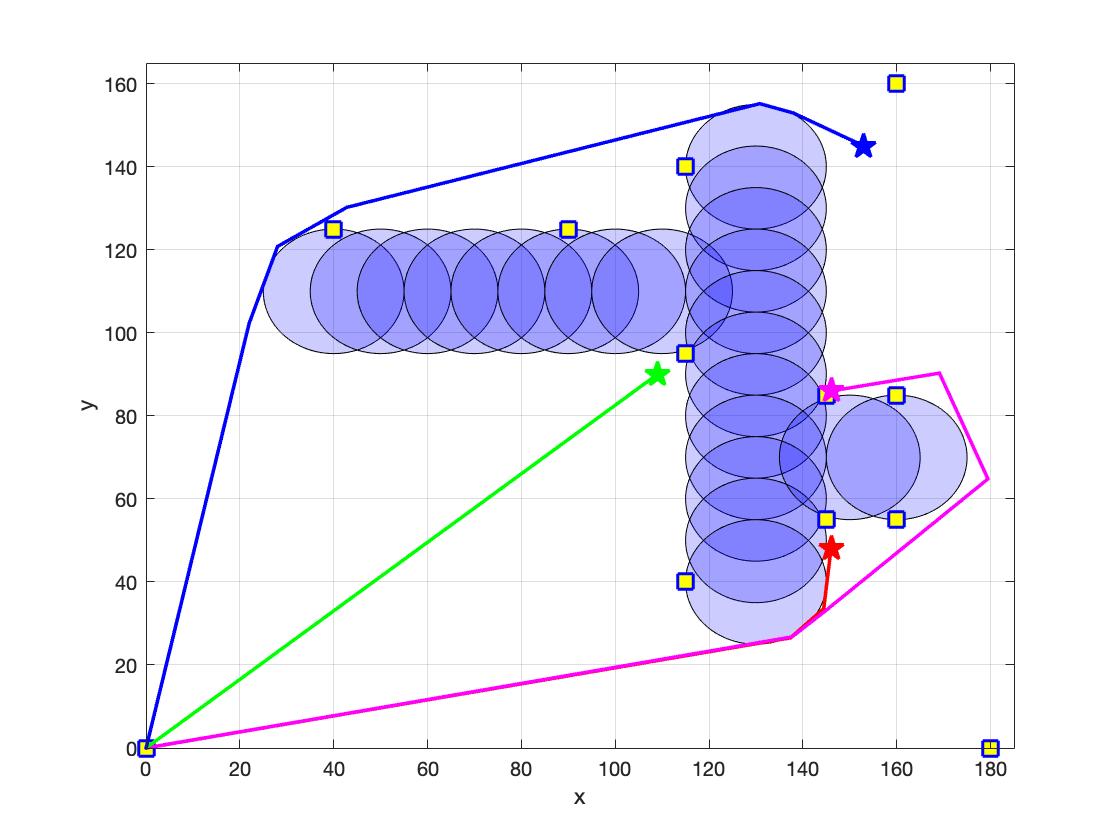} }}
  \subfloat[Deformed environment]{{\label{fig:Matlab-rotate}\includegraphics[width=4cm]{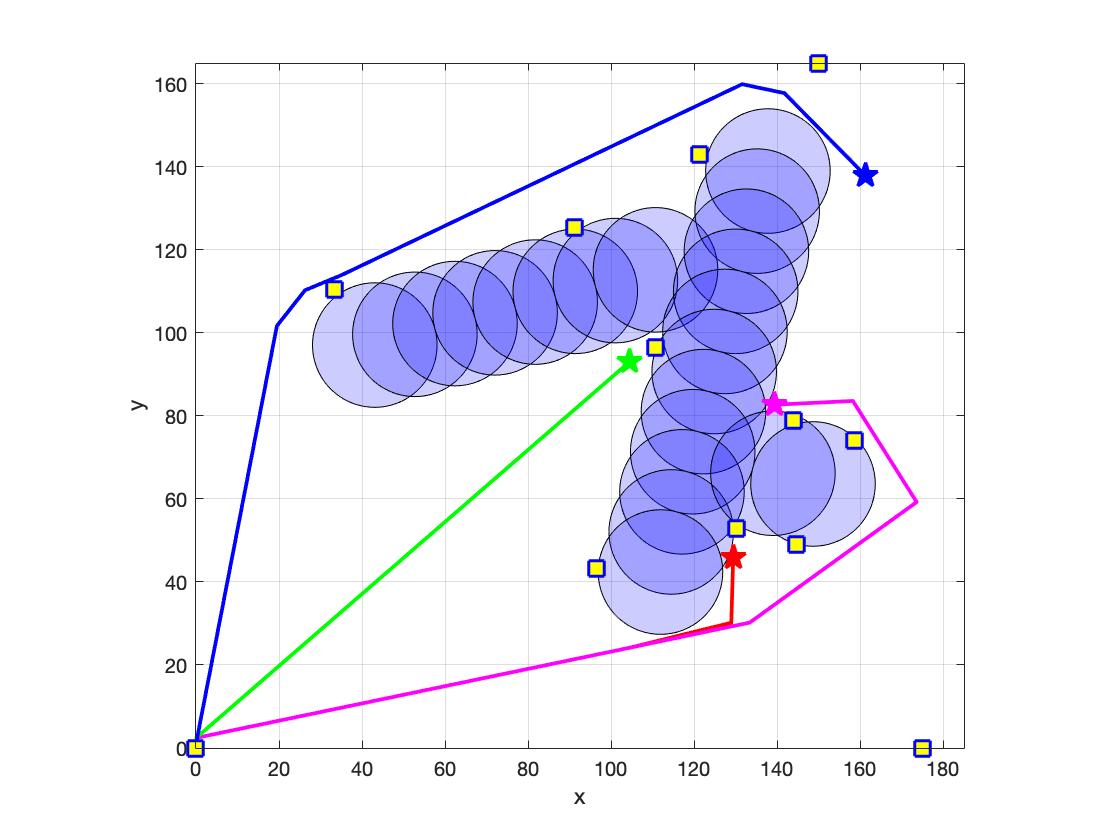} }}%
  \caption{Simulated trajectories for various start points for the original environment (Fig.~\ref{fig:Matlab-org}) and for a deformed version of the same environment (Fig.~\ref{fig:Matlab-rotate}). The start points for each trajectory are represented by star markers.}
  \label{fig:Matlab-sim}
\end{figure}
\section{SIMULATION AND EXPERIMENTAL RESULTS}
To assess the effectiveness of the proposed algorithm, we run a set of validation using both MATLAB simulations and experiments using ROS on a Create 2 robot by iRobot~\cite{iCreate}.
While the optimization problem guarantees exponential convergence of the robot to the stabilization point, in these experiments the velocity control input $u$ has been normalized to achieve constant velocities along the edges of the trees.
\begin{figure}[t]
  \centering
  \subfloat[Creat 2 iRobot]{\label{fig:robot}{\includegraphics[width=3.5cm]{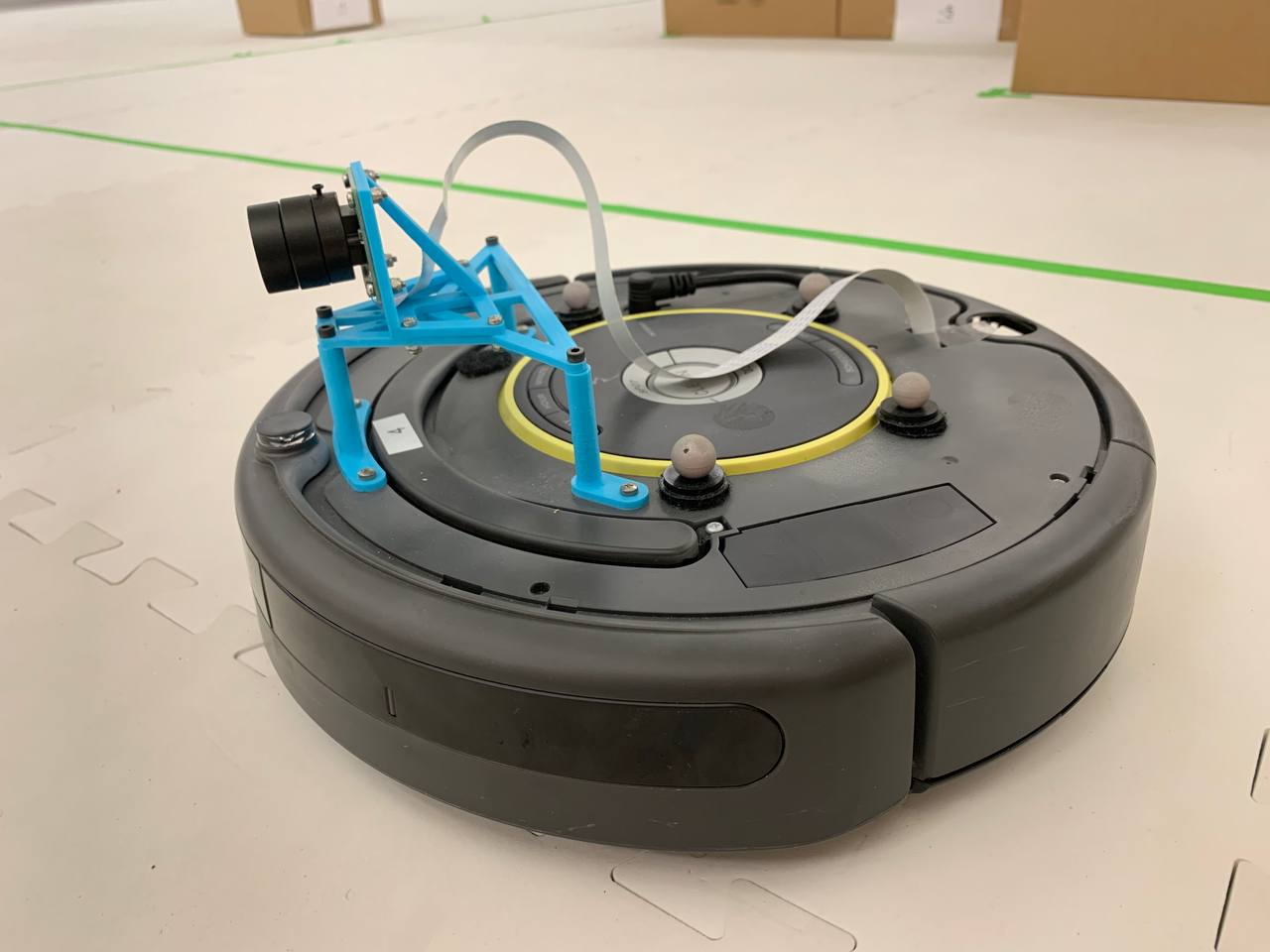} }}
  \subfloat[AprilTag]{{\label{fig:apriltag-figure}\includegraphics[width=2cm]{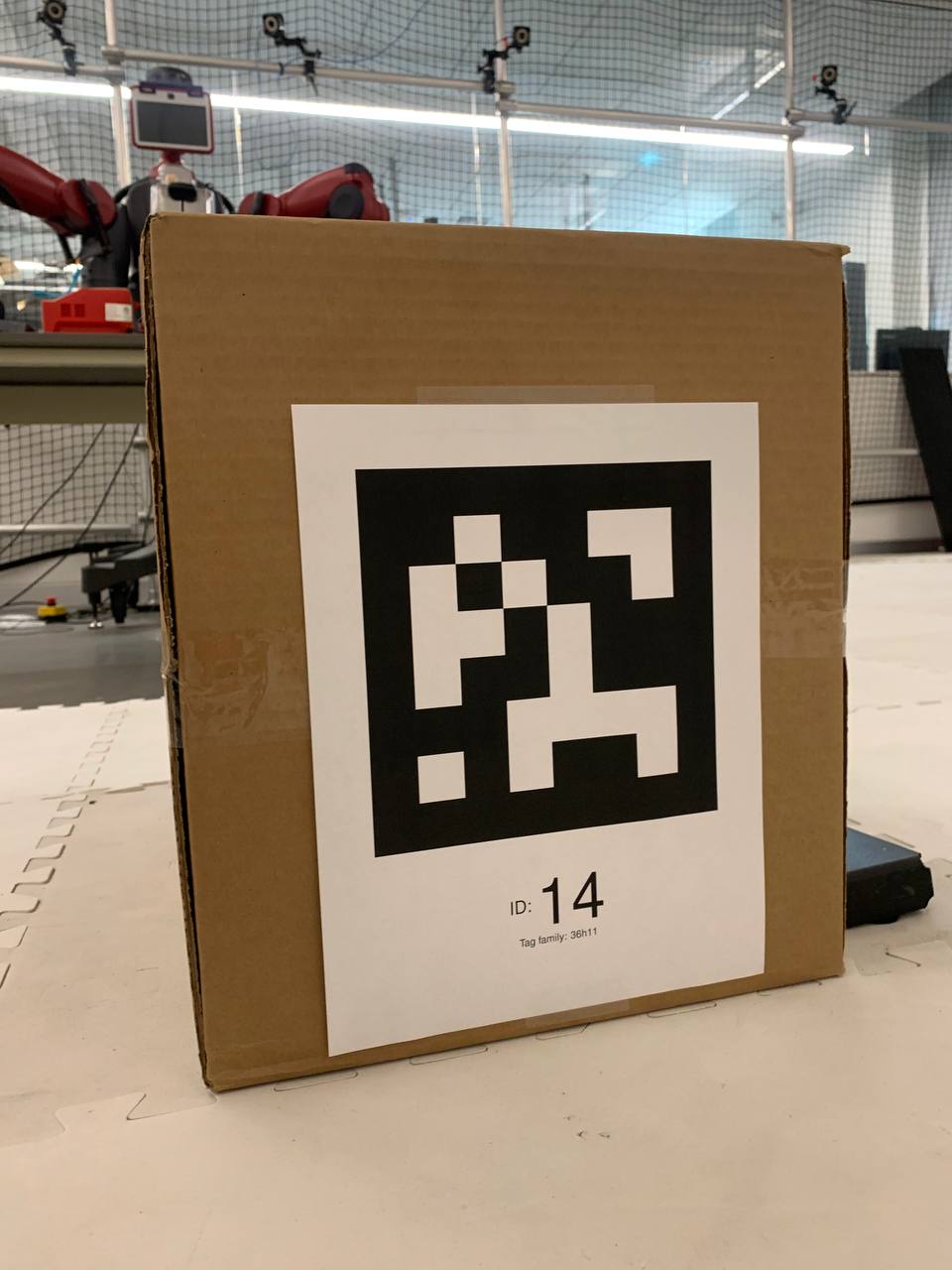} }}%
  \caption{The Create 2 robot used for the experiments is shown in Fig.~\ref{fig:robot}. We use AprilTags (Fig.~\ref{fig:apriltag-figure}) as the landmarks for the algorithm.}
  \label{fig:}
\end{figure}
\begin{figure}[t]
  \centering
  \subfloat[Real world environment]{\label{fig:original-env}{\includegraphics[width=2.75cm]{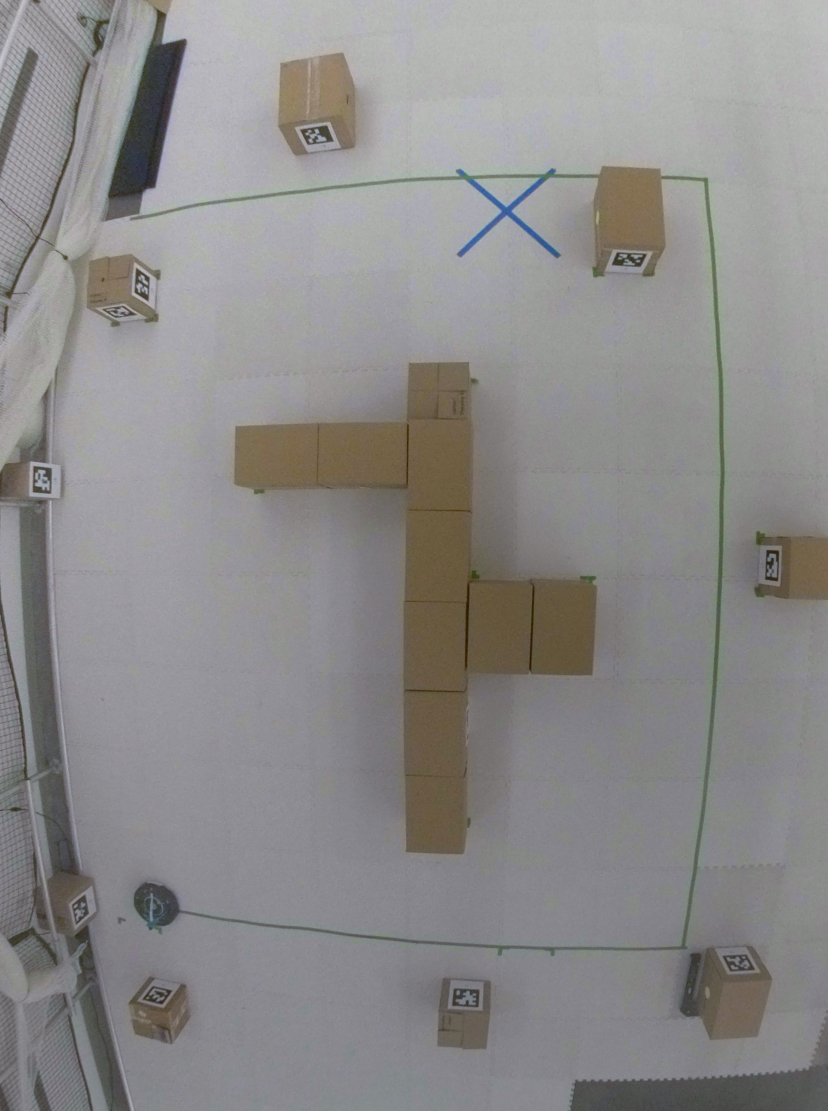} }}
  \subfloat[Deformed environment]{{\label{fig:rotated-env}\includegraphics[width=2.6cm]{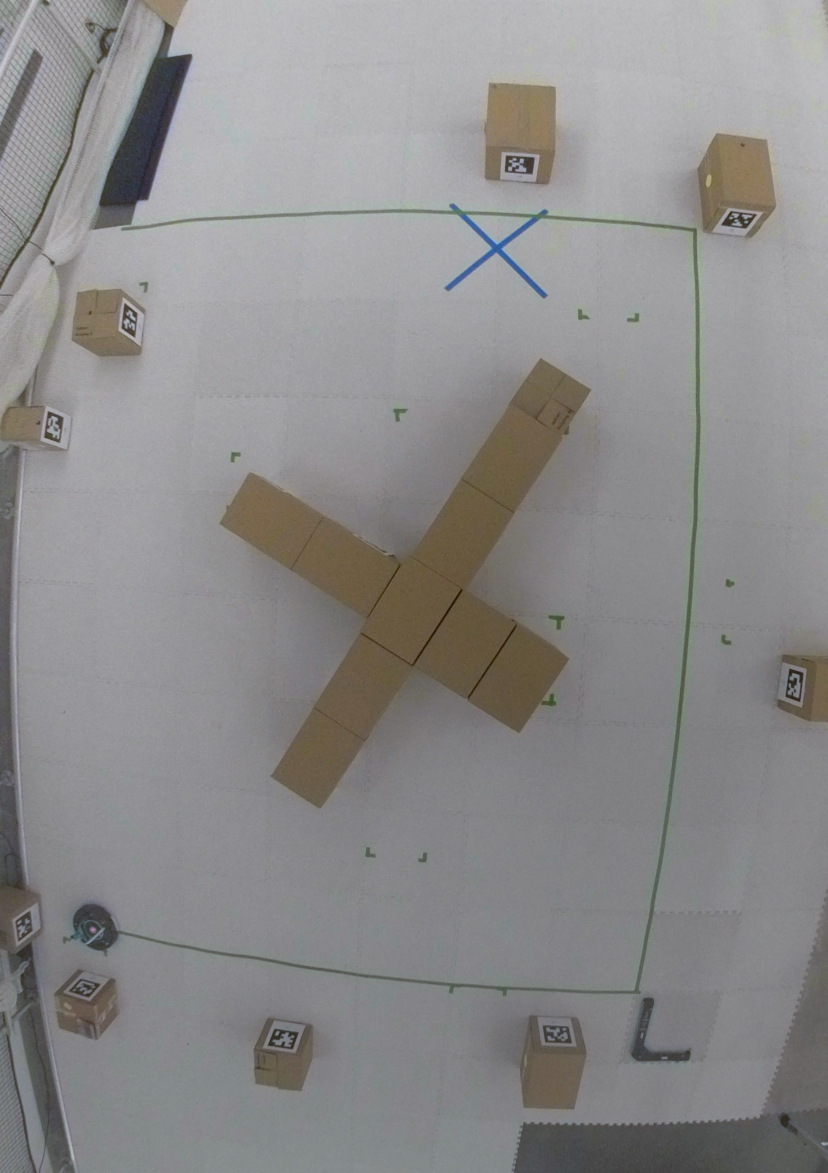} }}%
  \caption{On the left, Fig.~\ref{fig:original-env} shows the original environment used for our experiment. On the right, Fig.~\ref{fig:rotated-env} shows the deformed version of the environment.}
  \label{fig:real-example}
\end{figure}

\subsection{MATLAB Simulation}
The simulated MATLAB environment is presented in Fig.~\ref{fig:tree}, where the obstacles are represented by blue circles. To generate the \rrtstar{} we set the maximum number of iterations in \rrtstar{} to $1000$, and we choose $\eta=60$. 
The generated tree from \rrtstar{} and its simplified form are shown in Fig.~\ref{fig:tree-1} and Fig.~\ref{fig:tree-2} respectively. Then, we compute a controller for each edge of the simplified tree as described in Section~\ref{sec:tree-controller}. Fig.~\ref{fig:Matlab-sim} shows the resulting trajectories from four initial positions on two versions of the environment: one with the obstacles identical to the ones used during planning (Fig.~\ref{fig:Matlab-org}), and one with deformed obstacles (Fig.~\ref{fig:Matlab-rotate}); for the latter, also the landmarks have been modified accordingly. 
In all cases, the robot reaches the desired goal location by applying the sequence of controllers found during planning. Note that the deformed environment in Fig.~\ref{fig:Matlab-rotate} is successfully handled without replanning (i.e., by using the original controllers). This shows that our algorithm can be robust to (often very significant) deformations of the environment; however, there are also cases where, without replanning, the designed controllers might fail. Empirically, we noticed that there is a trade-off between obtaining shortest paths (that, by their nature, graze the obstacles) and the robustness of the controller; 
we plan to study this trade-off in future work.
\begin{figure}[b]
  \centering
  \subfloat[Original environment]{\label{fig:robot-original-trajectory}{\includegraphics[width=3.1cm]{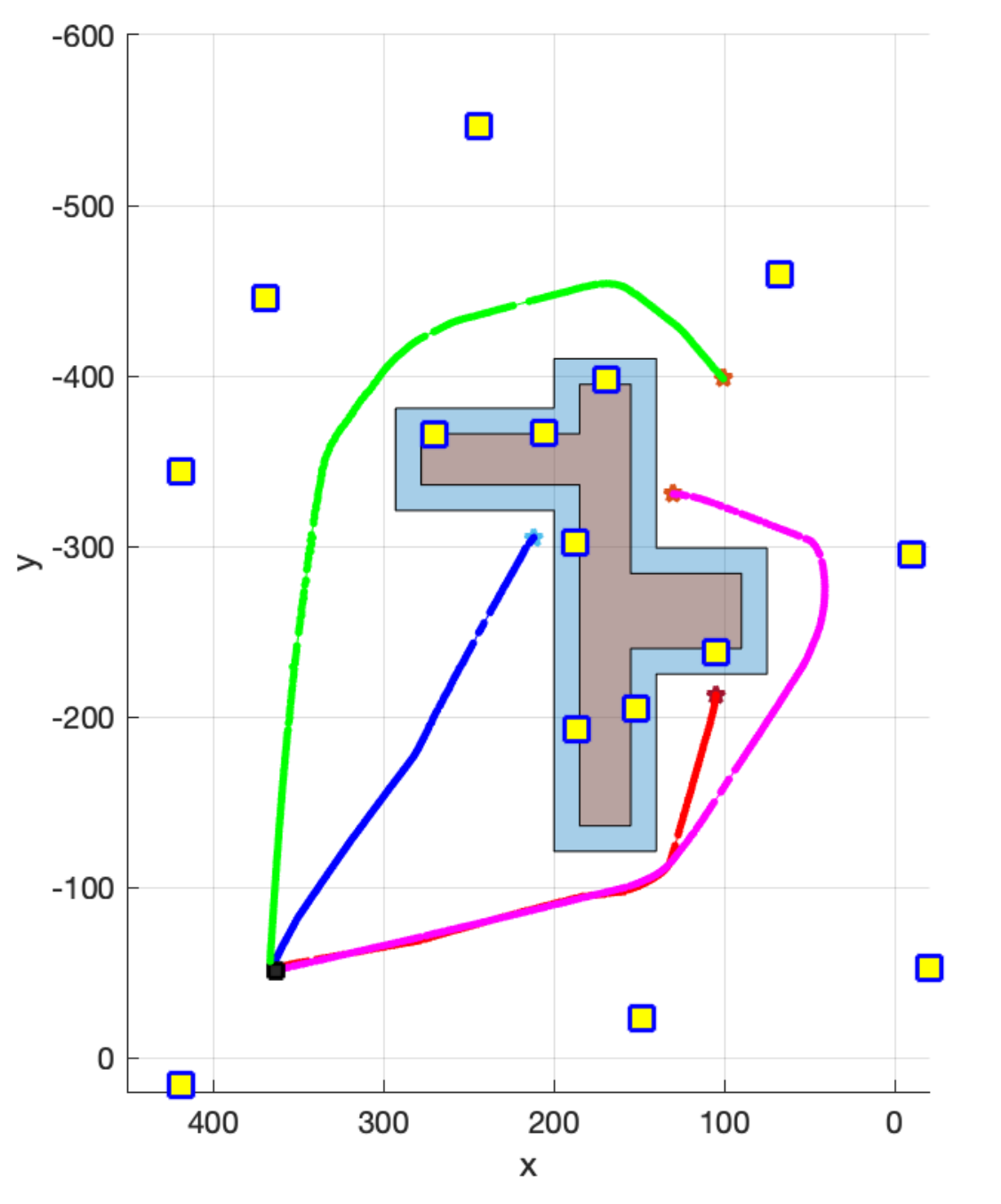} }} 
  \subfloat[Deformed environment]{{\label{fig:robot-deformed-trajectory}\includegraphics[width=3.1cm]{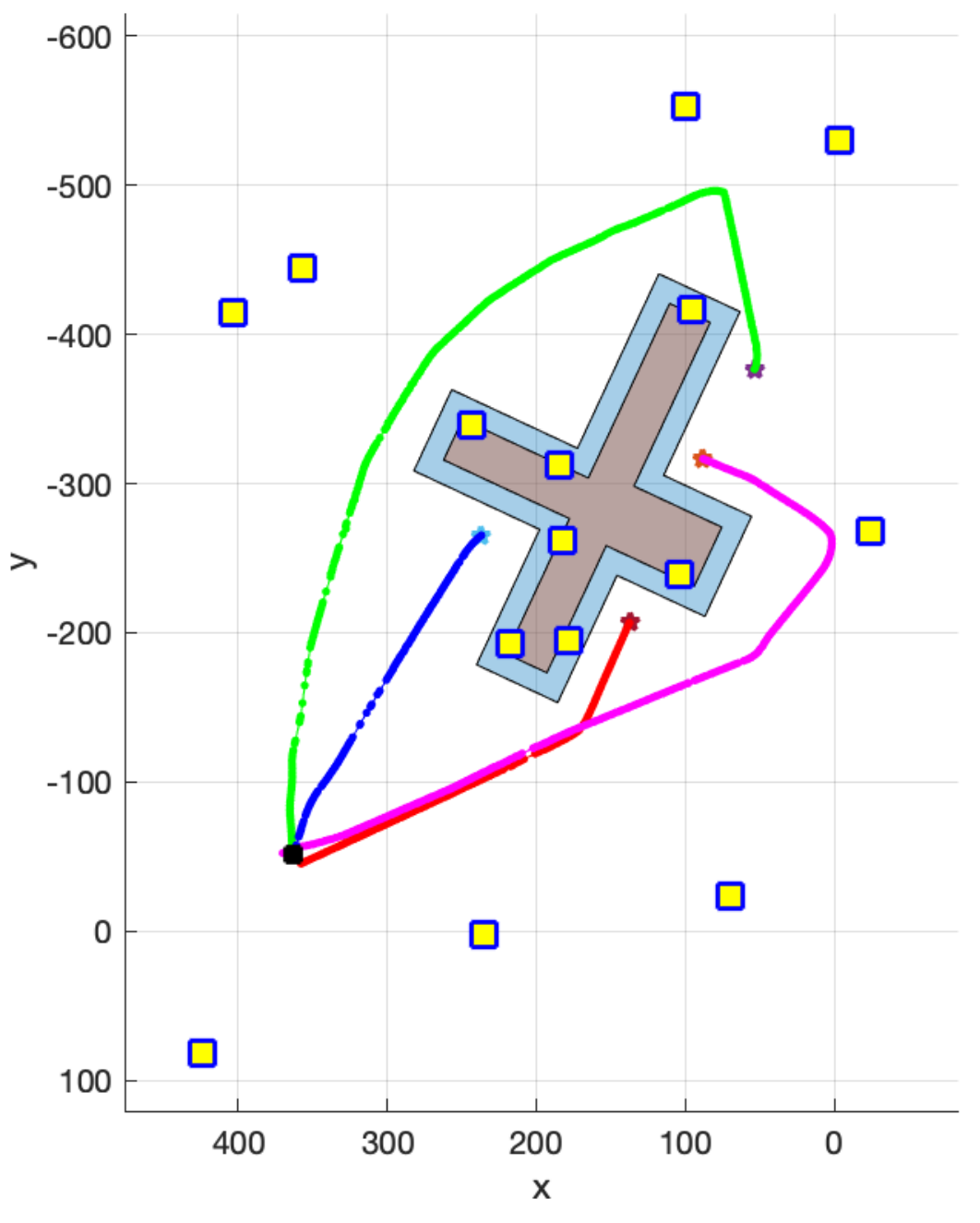} }}%
  \caption{Real trajectories followed by the Create for both the original and deformed environments. For all the tested start points, the robot converged to the expected goal position.}
  \label{fig:experiment-results}
\end{figure}

\subsection{iRobot Create 2 Experiment}
We further tested our algorithm on a Create 2 robot in lab environments that were similar to those used during the simulation. A bird's-eye view of the experimental setup is shown in Fig.~\ref{fig:real-example}. 
The robot is equipped with a calibrated onboard Arducam for Raspberry Pi camera \cite{arducam}, and we use OptiTrack motion capture system with 44 infra-red cameras to collect ground truth position information (the motion capture system is not used by our controller).
The landmarks are represented by fiducials (AprilTags~\cite{Wang2016}), and are placed at known positions and orientations with respect to the reference frame of the motion capture, using unique codes for data association. Our implementation is based on the Robot Operating System (ROS, \cite{ROS}).

There are three practical considerations that need to be taken into account in the implementation. First, due to the limited field of view of the camera, we use Proposition~\ref{prop:limited field view} to compute the controller based on one of the fiducials detected by the camera at each time instant. Second, the approach presented in the previous sections implicitly assumes that the robot has access to the measurements $Y$ in a frame which is rotationally aligned with the world reference frame.
To satisfy this assumption, the measured displacement of an AprilTag with respect to the Create in world coordinates $\left({}^{W}t_{AT-C}\right)$ is computed as:
\begin{equation}
  {}^{W}\mathbf{t}_{AT-C} = {}^{W}\mathbf{R}_{AT} \cdot \left( {}^{C}\mathbf{R}_{AT} \right)^{\top} \cdot {}^{C}\mathbf{t}_{AT-C},
\end{equation}
where ${}^{C}\mathbf{t}_{AT-C}$ is the measured displacement in Create coordinates, ${}^{C}\mathbf{R}_{AT}$ is the measured orientation of the AprilTag with respect to the Create, and ${}^{W}\mathbf{R}_{AT}$ is the \emph{a priori} known orientation of the AprilTag with respect to the world reference frame. Finally, $y_i\defeq{}^{W}\mathbf{t}_{AT-C}$ is used to compute the next control input $u$ following equation \eqref{u_limited}.

Finally, previous sections assumed a linear dynamical model for the robot, while the Create 2 has a unicycle dynamics. We map the original 2D input $u$ to a linear velocity $u_x$ along the $x$ axis of the robot and an angular velocity $\omega_z$ around the $z$ axis of the robot using a rather standard low-level controller:

\begin{align}
u_x = \small{\dfrac{\alpha}{\left\| u \right\|}
\begin{bmatrix}
    \cos{\varphi} \\ \sin{\varphi}
  \end{bmatrix}}\transpose u,&&
\omega_z = \small{\frac{\beta}{\left\| u \right\|} \bmat{0\\0\\1}\transpose \left(
                       \begin{bmatrix}
                         \cos{\varphi} \\ \sin{\varphi} \\ 0
                       \end{bmatrix}
  \times \begin{bmatrix}
    u \\ 0
  \end{bmatrix} \right)}
\end{align}
where $\varphi$ is the instantaneous yaw rotation of the robot with respect to the world reference frame, which is extracted from ${}^{C}\mathbf{R}_{AT}$ and ${}^{W}\mathbf{R}_{AT}$, and $\times$ represent the 3-D cross product; $\alpha$ and $\beta$ are user-defined scalar gains, 0.1 and 0.5 respectively.

Fig.~\ref{fig:experiment-results} depicts the real robot trajectories. Both in the original and deformed environments, the robot followed the edges of the \rrtstar{} tree and reached the expected goal for all starting positions and with the same control gains, despite the fact that the measurements were obtained with vision alone, and despite the different dynamics of the robot.

\section{CONCLUSIONS AND FUTURE WORKS}\label{sec:conclusions}
In this work, we introduced a new approach to integrate the high-level \rrtstar{} path planning with a low-level controller. We represented the environment via a simplified tree graph by implementing a modified sampling-based \rrtstar{} algorithm. We defined convex cells around the nodes of the tree and formulated a min-max robust Linear Program with CLF and CBF constraints to guarantee the stability and safety of the system. We built a robust output feedback controller for each cell which takes relative displacement measurements between a set landmarks positions and position of the robot as an input. We addressed the limited filed of view of the robot issue by reformulating a controller based on the visible landmarks. We validated our approach on both simulation environment and real-world environment and represented the robustness of our algorithm by applying the controller to a significantly deformed environment without replanning. We plan to prove the robustness of our algorithm theoretically and define the conditions of robustness of the controller in our future work. Furthermore, we plan to study the trade-off between optimal navigation and the robustness of the controller.


\bibliographystyle{IEEEtran}
\bibliography{references.bib,biblio/IEEEFull.bib,biblio/planning.bib,biblio/hardware.bib,biblio/control.bib}

\begin{thebibliography}{10}
\providecommand{\url}[1]{#1}
\csname url@samestyle\endcsname
\providecommand{\newblock}{\relax}
\providecommand{\bibinfo}[2]{#2}
\providecommand{\BIBentrySTDinterwordspacing}{\spaceskip=0pt\relax}
\providecommand{\BIBentryALTinterwordstretchfactor}{4}
\providecommand{\BIBentryALTinterwordspacing}{\spaceskip=\fontdimen2\font plus
\BIBentryALTinterwordstretchfactor\fontdimen3\font minus
  \fontdimen4\font\relax}
\providecommand{\BIBforeignlanguage}[2]{{%
\expandafter\ifx\csname l@#1\endcsname\relax
\typeout{** WARNING: IEEEtran.bst: No hyphenation pattern has been}%
\typeout{** loaded for the language `#1'. Using the pattern for}%
\typeout{** the default language instead.}%
\else
\language=\csname l@#1\endcsname
\fi
#2}}
\providecommand{\BIBdecl}{\relax}
\BIBdecl

\bibitem{kavraki1996probabilistic}
L.~E. Kavraki, P.~Svestka, J.-C. Latombe, and M.~H. Overmars, ``Probabilistic
  roadmaps for path planning in high-dimensional configuration spaces,''
  \emph{IEEE transactions on Robotics and Automation}, vol.~12, no.~4, pp.
  566--580, 1996.

\bibitem{rrt}
S.~M. LaValle, ``Rapidly-exploring random trees: A new tool for path
  planning,'' Iowa State University, Tech. Rep., 1998.

\bibitem{lavalle2001randomized}
S.~M. LaValle and J.~J. Kuffner~Jr, ``Randomized kinodynamic planning,''
  \emph{The international journal of robotics research}, vol.~20, no.~5, pp.
  378--400, 2001.

\bibitem{karaman2011sampling}
S.~Karaman and E.~Frazzoli, ``Sampling-based algorithms for optimal motion
  planning,'' \emph{The international journal of robotics research}, vol.~30,
  no.~7, pp. 846--894, 2011.

\bibitem{lavalle2006planning}
S.~M. LaValle, \emph{Planning algorithms}.\hskip 1em plus 0.5em minus
  0.4em\relax Cambridge university press, 2006.

\bibitem{renganathan2020towards}
V.~Renganathan, I.~Shames, and T.~H. Summers, ``Towards integrated perception
  and motion planning with distributionally robust risk constraints,''
  \emph{arXiv preprint arXiv:2002.02928}, 2020.

\bibitem{kuwata2008motion}
Y.~Kuwata, J.~Teo, S.~Karaman, G.~Fiore, E.~Frazzoli, and J.~How, ``Motion
  planning in complex environments using closed-loop prediction,'' in
  \emph{AIAA Guidance, Navigation and Control Conference and Exhibit}, 2008, p.
  7166.

\bibitem{positiveInvariant}
F.~Borrelli, A.~Bemporad, and M.~Morari, \emph{Predictive control for linear
  and hybrid systems}.\hskip 1em plus 0.5em minus 0.4em\relax Cambridge
  University Press, 2017.

\bibitem{weiss2017motion}
A.~Weiss, C.~Danielson, K.~Berntorp, I.~Kolmanovsky, and S.~Di~Cairano,
  ``Motion planning with invariant set trees,'' in \emph{2017 IEEE Conference
  on Control Technology and Applications (CCTA)}.\hskip 1em plus 0.5em minus
  0.4em\relax IEEE, 2017, pp. 1625--1630.

\bibitem{tedrake2009lqr}
R.~Tedrake, ``Lqr-trees: Feedback motion planning on sparse randomized trees,''
  \emph{MIT Press}, 2009.

\bibitem{ames2014control}
A.~D. Ames, J.~W. Grizzle, and P.~Tabuada, ``Control barrier function based
  quadratic programs with application to adaptive cruise control,'' in
  \emph{53rd IEEE Conference on Decision and Control}.\hskip 1em plus 0.5em
  minus 0.4em\relax IEEE, 2014, pp. 6271--6278.

\bibitem{hsu2015control}
S.-C. Hsu, X.~Xu, and A.~D. Ames, ``Control barrier function based quadratic
  programs with application to bipedal robotic walking,'' in \emph{2015
  American Control Conference (ACC)}.\hskip 1em plus 0.5em minus 0.4em\relax
  IEEE, 2015, pp. 4542--4548.

\bibitem{borrmann2015control}
U.~Borrmann, L.~Wang, A.~D. Ames, and M.~Egerstedt, ``Control barrier
  certificates for safe swarm behavior,'' \emph{IFAC-PapersOnLine}, vol.~48,
  no.~27, pp. 68--73, 2015.

\bibitem{Mahroo}
M.~Bahreinian, E.~Aasi, and R.~Tron, ``Robust planning and control for
  polygonal environments via linear programming,'' \emph{2020 IEEE American
  Control Conference (ACC)}, 2020.

\bibitem{Isidori:book95}
A.~Isidori, \emph{Nonlinear control systems}.\hskip 1em plus 0.5em minus
  0.4em\relax Springer Science \& Business Media, 1995.

\bibitem{zcbf1}
X.~Xu, P.~Tabuada, J.~W. Grizzle, and A.~D. Ames, ``Robustness of control
  barrier functions for safety critical control,'' \emph{IFAC-PapersOnLine},
  vol.~48, no.~27, pp. 54--61, 2015.

\bibitem{nguyen2016exponential}
Q.~Nguyen and K.~Sreenath, ``Exponential control barrier functions for
  enforcing high relative-degree safety-critical constraints,'' in \emph{2016
  American Control Conference (ACC)}.\hskip 1em plus 0.5em minus 0.4em\relax
  IEEE, 2016, pp. 322--328.

\bibitem{nash2007theta}
A.~Nash, K.~Daniel, S.~Koenig, and A.~Felner, ``Theta\^{}*: Any-angle path
  planning on grids,'' in \emph{AAAI}, vol.~7, 2007, pp. 1177--1183.

\bibitem{Choset:book05}
H.~M. Choset, K.~M. Lynch, S.~Hutchinson, G.~Kantor, W.~Burgard, L.~Kavraki,
  and S.~Thrun, \emph{Principles of robot motion: theory, algorithms, and
  implementation}.\hskip 1em plus 0.5em minus 0.4em\relax MIT press, 2005.

\bibitem{latombe2012robot}
J.-C. Latombe, \emph{Robot motion planning}.\hskip 1em plus 0.5em minus
  0.4em\relax Springer Science \& Business Media, 2012, vol. 124.

\bibitem{iCreate}
iRobot, ``{Create 2 Programmable Robot},''
  \url{https://edu.irobot.com/what-we-offer/create-robot}.

\bibitem{arducam}
Arducam, ``{Camera for Raspberry Pi},''
  \url{https://www.arducam.com/docs/cameras-for-raspberry-pi/}.

\bibitem{Wang2016}
J.~Wang and E.~Olson, ``{AprilTag 2: Efficient and robust fiducial
  detection},'' in \emph{2016 IEEE/RSJ International Conference on Intelligent
  Robots and Systems (IROS)}.\hskip 1em plus 0.5em minus 0.4em\relax IEEE, oct
  2016, pp. 4193--4198.

\bibitem{ROS}
\BIBentryALTinterwordspacing
{Stanford Artificial Intelligence Laboratory et al.}, ``Robotic operating
  system.'' [Online]. Available: \url{https://www.ros.org}
\BIBentrySTDinterwordspacing

\end{thebibliography}
\end{document}